\documentclass{article}

\usepackage{latexsym,amssymb}

\newcommand{\set}[1]{\{#1\}}
\newcommand{\setof}[2]{\{#1\,|\:#2\}}
\newcommand{\card}[1]{|#1|}
\newcommand{\tuple}[1]{\langle #1 \rangle}
\newcommand{\inv}[1]{#1^{-1}}

\newcommand{\eF}[1]{\mathbb{F}_{#1}}
\newcommand{\Sym}[1]{\mathrm{Sym}(#1)}

\newcommand{\gen}[1]{\lceil #1 \rceil}
\newcommand{\orbp}[1]{\mathcal{OP}(#1)}
\newcommand{\orbpG}{\mathcal{P}}
\newcommand{\QP}{\mathcal{Q}}
\newcommand{\finer}{\sqsubseteq}
\newcommand{\merge}{\sqcup}

\newcommand{\botp}[1]{\bot_{#1}}
\newcommand{\topp}[1]{\top\!_{#1}}

\newcommand{\cstr}{\mathcal{C}}
\newcommand{\cstrat}[1]{\cstr(#1)}
\newcommand{\affbij}[1]{\phi_{#1}}
\newcommand{\affbijof}[2]{\affbij{#1}(#2)}

\newcommand{\matrvar}[1]{\mathcal{M}_{#1}}
\newcommand{\cstrs}[1]{\mathcal{C}_{#1}}
\newcommand{\cstrsat}[2]{\cstrs{#1}(#2)}
\newcommand{\cstrps}[1]{\mathcal{C'}_{#1}}
\newcommand{\cstrpsat}[2]{\cstrps{#1}(#2)}

\newcommand{\supspc}{F}
\newcommand{\supbasis}{\mathbf{f}}
\newcommand{\basevec}[1]{f_{#1}}
\newcommand{\cstrspc}{E}

\newcommand{\buildbasisfn}{\mathrm{B}}
\newcommand{\buildbasis}[3]{\buildbasisfn(#1,\,#2,\,#3)}

\newcommand{\restr}[2]{#1{\scriptstyle |}_{#2}}
\newcommand{\kron}[1]{\mathrm{\delta}_{#1}}

\newcommand{\colmat}[2]{#1_{#2}}
\newcommand{\transpose}[1]{{^t}#1}

\newenvironment{example}{\vspace{1ex}\noindent\emph{Example.}}{\hspace{0.5em}$\Diamond$\vspace{1ex}}

\newtheorem{definition}{Definition}

\newtheorem{lemma}{Lemma}
\newtheorem{theorem}[lemma]{Theorem}
\newtheorem{corollary}[lemma]{Corollary}

\newenvironment{proof}{\vspace{.1ex}\noindent\emph{Proof.}}{\hspace{\fill}$\Diamond$\vspace{.5ex}\newline}

\title{Solving Linear Constraints in Elementary Abelian $p$-Groups of Symmetries}

\author{Thierry Boy de la Tour \& Mnacho Echenim\\
\small Laboratoire d'Informatique de Grenoble, CNRS - Grenoble INP\\
\small    B\^{a}timent IMAG C - 220 rue de la Chimie\\
\small    F-38400 Saint-Martin-d'H\`{e}res\\
\small email: \texttt{thierry.boy-de-la-tour@imag.fr}, \texttt{mnacho.echenim@imag.fr}}
\date{}

\begin{document}

\maketitle

\begin{abstract}
Symmetries occur naturally in CSP or SAT problems and are not very
difficult to discover, but using them to prune the search space tends
to be very challenging. Indeed, this usually requires finding specific
elements in a group of symmetries that can be huge, and the problem of
their very existence is NP-hard. We formulate such an existence
problem as a constraint problem on one variable (the symmetry to be
used) ranging over a group, and try to find restrictions that may be
solved in polynomial time. By considering a simple form of constraints
(restricted by a cardinality $k$) and the class of groups that have
the structure of $\eF{p}$-vector spaces, we propose a partial
algorithm based on linear algebra. This polynomial algorithm always
applies when $k=p=2$, but may fail otherwise as we prove the problem
to be NP-hard for all other values of $k$ and $p$. Experiments show
that this approach though restricted should allow for an efficient use
of at least some groups of symmetries.
We conclude with a few directions to be explored to efficiently
solve this problem on the general case.

\end{abstract}
keywords: symmetries, linear algebra, complexity

\section{Introduction}\label{sec-intro}

Symmetries are permutations of input symbols that, when applied to an
instance of a computational problem, leave its solution
invariant. Since na\"{\i}ve algorithms reproduce these symmetries in
their search space, it is tempting to use them as a pruning device.  A
typical example is the pigeon-hole problem in propositional logic: its
symmetries show that any pigeon can be swapped with any other (and
similarly for the holes), and therefore play equivalent r\^{o}les.

Since invariance is stable by composition, the set of
symmetries of an instance forms a permutation group. This means that
the information they provide has a mathematical structure that should
yield nice computing properties (even if they occur only at the
meta-level). In fact, some techniques from computational group theory
have been employed to discover symmetries and to use them.

However, it seems that algorithms can use symmetries in a straightforward
way only if their search space preserves the structure of the group of
symmetries in some sense (subtrees should somehow correspond to
subgroups). This is usually not the case for the most efficient algorithms
developed in the field of AI. Hence many different methods have been
developed in order to prune the search space with symmetries, either by
designing special algorithms or by modifying instances through
symmetry-breaking. Much work has been devoted to this subject, see
e.g. \cite{Walsh10} and the references therein.

One feature common to these methods is that they ideally assume the
ability to produce symmetries that have particular properties,
suitable for the pruning scheme. But this problem also happens to be
NP-hard, which explains why symmetries that do not result in the best
pruning may be used. The time spent on searching and using symmetries
does not always pay off. This suggests that the group structure may
not be sufficient to induce enough computational properties to ensure
efficient pruning.

Our aim in this paper is to investigate ways of finding suitable symmetries
in polynomial time. To this purpose we first formulate this search problem
by a language of constraints as simple as possible. This is the topic of
Section \ref{sec-cstr}. We then consider restrictions of this search
problem that confer deeper mathematical structure to the groups. The idea
is to transform the constraints into \emph{linear equations}, which would
then be easy to solve by means of basic computer algebra. To the best of
our knowledge, although restrictions to vector spaces have already been
considered in the literature, such a transformation represents a novel
approach. This means that we need to assume that the groups are also vector
spaces, and to develop ways of efficiently working with symmetries as
vectors (i.e., essentially of computing their coordinates in a suitable
basis). This is developed in Sections \ref{sec-algbr} and
\ref{sec-solvlincstr}, leading to a polynomial algorithm that solves
so-called linear constraints. We will then see in Section \ref{sec-NPC}
that, even with the simple constraints and vector spaces, the search
problem remains NP-hard in most cases. Experiments in Section
\ref{sec-expe} illustrate the efficiency of this polynomial algorithm on
random samples of linear constraints, compared to a general purpose
algorithm.  We suggest in the conclusion a few directions for
using this approach in a wider setting.

\section{Definitions}\label{sec-def}

We do not recall the most basic definitions and notations from group
theory or linear algebra, such as cycles or bases, which can be found
in standard textbooks, e.g. \cite{mDUM99a}, except in order to settle
notations.

Given a finite set $A$, we denote by $\Sym{A}$ the group of
permutations of $A$. If $g_1,\ldots,g_m$ are permutations of $A$, we
denote by $\gen{g_1,\ldots,g_m}$ the subgroup of $\Sym{A}$ generated
by these permutations. For $a\in A$ and $g,g'\in\Sym{A}$, the image of
$a$ by $g$ is denoted by $a^g$, and the composition of permutations
$g'\circ g$ by $gg'$, so that $a^{gg'} = (a^g)^{g'}$. From a
computational point of view, it is obvious that the product $gg'$ can
be performed in time linear in $\card{A}$. The \emph{order} of $g$ is
the smallest positive integer $n$ such that $g^n$ is the identity.

Let $G$ be a permutation group on $A$, the \emph{orbit} of $a$ in $G$
(or $G$-orbit of $a$) is $a^G = \setof{a^g}{g\in G}$. The set of
$G$-orbits forms a partition of $A$, denoted by $\orbp{G}$. The group
$G$ is \emph{transitive} if it has only one orbit (i.e., $\orbp{G} =
\set{A}$). It is easy to see that $\orbp{\gen{g}}$ can be obtained
from the cycles of $g$, e.g. if $A=\set{1,\ldots,6}$ then
$\orbp{\gen{(1\ 2)(3\ 4\ 5)}} = \set{\set{1,2},\set{3,4,5},\set{6}}$.

The refinement order on partitions of $A$ ($P\finer
P'$ iff $\forall O\in P, \exists O'\in P'$ s.t. $O\subseteq O'$) is a
complete lattice; the least upper bound $P\merge P'$ is obtained by merging
the non-disjoint elements of $P$ and $P'$. The smallest partition is
$\botp{A} = \setof{\set{a}}{a\in A}$ and the greatest is $\topp{A} =
\set{A}$. Given two permutation groups $G$ and $G'$ on $A$,
$\orbp{\gen{G\cup G'}} = \orbp{G}\merge\orbp{G'}$ (see
\cite[chap. 7]{GB}). Hence, starting with $m$ generators $g_1,\ldots,g_m$
the orbit partition $\orbp{\gen{g_1,\ldots,g_m}} = \bigsqcup_{i=1}^m
\orbp{\gen{g_i}}$ can be computed in time polynomial in $m$ and $\card{A}$.

A group $G$ is an \emph{elementary Abelian} $p$-group if it is Abelian
and its non-trivial elements have order $p$, a prime number.  It is
simple to test this property on the generators of a group: $G$ is an
elementary Abelian $p$-group iff its generators commute and have order
$p$. If this is the case we adopt the additive notation,
e.g. $(a\ b)+(c\ d) = (a\ b)(c\ d)$, $2(a\ b\ c) = (a\ b\ c)^2 =
(a\ c\ b)$ and 0 is the identity. By considering multiplication by an
integer as an external product on the set of integers modulo $p$, we
confer to $G$ the structure of an $\eF{p}$-vector space. Conversely,
every $\eF{p}$-vector space is isomorphic to an elementary Abelian
permutation $p$-group.

Since the class of elementary Abelian $p$-groups is closed under
homomorphic images and for all $O\in\orbp{G}$ the restriction to $O$
operator is a morphism from $G$ to $\Sym{O}$, if $G$ is an $\eF{p}$-vector
space then so is its image $\restr{G}{O} = \setof{\restr{g}{O}}{g\in G}$
(although it may not be a subgroup of $G$).
The groups $\restr{G}{O}$ are the \emph{transitive constituents} of $G$.

In the sequel, unless stated otherwise, $a,b,c$ denote members of
$A$, $u,v,w$ denote vectors (i.e., permutations on $A$), $\mathbf{h}$
and $\supbasis$ denote bases of vector spaces, and the $x_i$'s
coordinates in vector spaces (i.e., integers modulo $p$). 

\section{Constraints on symmetries}\label{sec-cstr}

As mentioned above, techniques designed to compute generators for a group
of symmetries are well-known. Basically, they consist in building a graph
with the elements to be permuted as vertices, and with edges and labels
that depend on the instance, so that its automorphism group is the expected
group of symmetries. Building this graph is usually straightforward and its
automorphism group can then be computed with e.g. the well-known program
\texttt{nauty} \cite{McKay90}. Though not polynomial, this algorithm has a
low average complexity and is very efficient.

We thus assume given a group of symmetries that is specified by a
generating set of $m$ permutations $g_1,\ldots,g_m$ of a set $A$. The group
$G=\gen{g_1,\ldots,g_m}$ is the set in which a permutation g satisfying a
given constraint is searched for. In this section we define a language for
expressing constraints on $g$ that is both useful and simple. We begin with
an example.

We consider the problem \textsc{mmc} from \cite{CrawfordGLR96}. Given
a model $M$ of a formula built on propositional variables which are
linearly ordered (say, $x<y<z$), $^gM$ being the interpretation
defined by $^gM(a) = M(a^g)$, and given a group $G$ by its generators,
the problem \textsc{mmc} consists in deciding whether there exists a
$g\in G$ such that $^gM<M$ (interpretations are ordered
lexicographically). This problem 
is required for computing symmetry-breaking predicates, and
is shown in \cite{CrawfordGLR96} to
be NP-complete. We can translate $^gM<M$ as: $M(x^g)M(y^g)M(z^g)
< M(x)M(y)M(z)$, and then as
\[\begin{array}{ll}
M(x^g)<M(x)\vee &\\
M(x^g)=M(x)\wedge[ & M(y^g)<M(y) \vee\\
& M(y^g)=M(y)\wedge M(z^g)<M(z)].
\end{array}\]
A translation into disjunctive normal form yields:
\[\begin{array}{ll}
& M(x^g)<M(x) \\
\vee & M(x^g)=M(x)\wedge M(y^g)<M(y) \\
\vee & M(x^g)=M(x)\wedge M(y^g)=M(y) \wedge M(z^g)<M(z).
\end{array}\]
If $X$ denotes $\setof{a}{M(a)<M(x)}$ then the first disjunct
translates into $x^g\in X$. Similarly, let $X'=\setof{a}{M(a)=M(x)}$
and $Y=\setof{a}{M(a)<M(y)}$, the second disjunct is $x^g\in X'\wedge
y^g\in Y$, etc. Thus the problem \textsc{mmc} can be expressed by
means of boolean combinations of \emph{atomic constraints} of the form
$x^{\sigma}\in X$, where $\sigma$ is the (unique) variable ranging in
$G$. Note that the negation of any atomic constraint can be expressed
as an atomic constraint
\footnote{This is no longer true if
  we break down $x^g\in\set{x_1,\ldots,x_n}$ into
  $x^g=x_1\vee\ldots\vee x^g=x_n$, hence atomic constraints of the
  form $x^g=y$ would not necessarily be simpler to handle.}
(with the
complement set), hence negation can be ruled out from the
language. Since the set of solutions of a
disjunction is the union of the solutions of each disjunct, we
focus on solving conjunctions of atomic constraints.
\begin{definition}
We address the computational problem \textsc{gc}, for Group
Constraint, that takes as input a set $A$ of cardinality $n$,
permutations $g_1,\ldots,g_m$ of $A$, and a constraint $\varphi$ which
is a conjunction of atomic formulas of the form $x^{\sigma}\in X$,
where $\sigma$ is a unique variable and $x\in A, X\subseteq A$. The
\emph{decision problem} $\textsc{gc}$ consist in checking for the
existence of a permutation $g\in \gen{g_1,\ldots,g_m}$ that satisfies
all the conjuncts in $\varphi$ (a conjunct $x^\sigma\in X$ is
\emph{satisfied} by $g$ if $x^g\in X$). The associated \emph{search
  problem} is the computation of such a $g$ if it exists.
\end{definition}

A number of transformations on the constraint that preserve
the set of solutions can be applied. These are:
\begin{eqnarray*}
x^{\sigma}\in X \wedge x^{\sigma}\in Y & \longrightarrow & x^{\sigma}\in X\cap Y, \\
\varphi & \longrightarrow & \varphi\ \wedge\  x^{\sigma}\in x^G.
\end{eqnarray*}
The correctness of the first transformation is obvious and that of the
second trivial since $x^g\in x^G$ holds for all $g\in G$. By applying
these transformations systematically (together with
associativity-commutativity of conjunction), assuming that
$A=\set{a_1,\ldots,a_n}$ it is always possible to transform any given
constraint into an equivalent constraint of the form $a_1^{\sigma}\in
A_1 \wedge \ldots \wedge a_n^{\sigma}\in A_n$, where $A_i\subseteq
a_i^G$ for all $1\leq i\leq n$. Computing this normal form is linear
in the length of $\varphi$ and $\card{A}$.

This normal form may be represented as a function $\cstr$ from $A$ to
$2^A$, where $\cstrat{a_i} = A_i$. The problem \textsc{gc} then
consists in finding a $g\in G$ such that $\forall a\in A, a^g\in
\cstrat{a}$, and we may consider $\cstr$ directly as an input of the
problem, so that an instance of \textsc{gc} is a tuple $\tuple{A,
g_1,\ldots, g_m, \cstr}$ (or $\tuple{A,G,\cstr}$ in
short). We say that $\cstr$ is a \emph{$k$-constraint} if $\forall
a\in A, \card{\cstrat{a}}\leq k$.

As mentioned above, in the sequel we only consider, for every prime
$p$, the restriction of \textsc{gc} to the instances where $G =
\gen{g_1,\ldots,g_m}$ is an elementary Abelian $p$-group, i.e., an
$\eF{p}$-vector space\footnote{Note that solving such constraints is
  much simpler than computing any lex-leader formula as in
  \cite{LuksR04}, where restrictions to vector spaces are already
  considered. Our constraints are therefore not relevant to the
  problem of building lex-leader formulas, and are not meant to be
  used in the symmetry-breaking scheme of \cite{CrawfordGLR96}.}.

\section{From permutations to linear algebra}\label{sec-algbr}

Before giving the technical details of the transformation from a
problem on permutations to a problem in linear algebra, we briefly
summarize the way we proceed.  

Since our approach is based on the transitive constituents
$\restr{G}{O}$ of $G$ (for $O\in\orbp{G}$), we first define a vector
space $\supspc$ that contains as subspaces both the group $G$ and the
$\restr{G}{O}$'s. We then prove a fundamental property of the
$\restr{G}{O}$'s and use it to realize a polynomial test of linear
dependence restricted to these subspaces. This is used to compute a
basis of each $\restr{G}{O}$ and to obtain the coordinates of any
$u\in\restr{G}{O}$ in this basis, in polynomial time. This directly
yields a basis of $\supspc$ and the coordinates of any $u\in\supspc$
in this basis. It is then standard to transform any linear variety of
$F$, such as $G$, as a set of linear equations using basic linear
algebra. Section \ref{sec-solvlincstr} will be devoted to applying
these techniques to the constraint $\cstr$.
In the sequel we write ${\orbpG}$ for the orbit partition $\orbp{G}$.

\subsection{The super-space $\supspc$}

The groups $G$ and the $\restr{G}{O}$'s are of course all included in
$\Sym{A}$, but $\Sym{A}$ is not elementary Abelian hence not a vector
space. Let $\supspc$ be the permutation group on $A$ generated by the
transitive constituents of $G$, i.e., $\supspc =
\gen{\bigcup_{O\in\orbpG}\restr{G}{O}}$.
\begin{lemma}\label{lem-E}
$\supspc$ is an $\eF{p}$-vector space that contains $G$ and
  $\displaystyle \supspc = \bigoplus_{O\in {\orbpG}}\restr{G}{O}$.
\end{lemma}
\begin{proof} 
  $\supspc$ is generated by the union of generating sets of the groups
  $\restr{G}{O}$, i.e., by the set $\setof{\restr{g_i}{O}}{1\leq i\leq
    m, O\in \orbpG}$. Since orbits are mutually disjoint and
  permutations on disjoint sets commute,
  $\restr{g_i}{O}\restr{g_j}{O'} = \restr{g_j}{O'}\restr{g_i}{O}$ if
  $O\neq O'$. Furthermore $\restr{g_i}{O}\restr{g_j}{O} =
  \restr{(g_ig_j)}{O} = \restr{(g_jg_i)}{O} =
  \restr{g_j}{O}\restr{g_i}{O}$ since $\restr{G}{O}$ is Abelian. Hence
  $\supspc$ is Abelian, and similarly its non-trivial elements have order
  $p$.

  We therefore use the additive notation in $\supspc$, which yields $\supspc =
  \sum_{O\in\orbpG}\restr{G}{O}$. For any $O\in\orbpG$, let
  $A'=A\setminus O$ and $\supspc_O = \sum_{O'\neq O}\restr{G}{O'}$, since
  $\supspc_O$ is a subgroup of $\Sym{A'}$, then $\restr{G}{O}\cap \supspc_O
  \subseteq \Sym{O}\cap\Sym{A'} = \set{0}$. This proves that $\supspc$ is
  the (internal) direct sum of the $\restr{G}{O}$'s, which is written
  $\supspc = \bigoplus_{O\in {\orbpG}}\restr{G}{O}$.

  It is clear that $G$ is a subspace of $\supspc$ since any $u\in G$ can be
  written as $u=\sum_{o\in {\orbpG}}\restr{u}{o}$ and hence belongs to
  $\supspc$. 
\end{proof}

We call $\supspc$ the \emph{super-space} of $G$. Since the sum in
Lemma \ref{lem-E} is direct, the decomposition of any vector $u$ in
$\supspc$ as a sum of elements of the $\restr{G}{O}$'s is unique, the
dimension $d$ of $\supspc$ is the sum of the dimensions $d_O$ of
$\restr{G}{O}$ for $O\in\orbpG$, and a basis for $\supspc$ can be
obtained by concatenating bases for the subspaces $\restr{G}{O}$.

\subsection{Orbits as affine spaces}\label{sec-affspc}

For all $O\in\orbpG$, the transitive constituent $\restr{G}{O}$ is of
course transitive in $O$. This trivial fact yields a fundamental property:

\begin{lemma}$(O,\restr{G}{O})$ is an affine space.
\end{lemma}
\begin{proof}
We define the external sum $+: O\times \restr{G}{O} \rightarrow O$,
for all $a\in O$ and $u\in \restr{G}{O}$, by $a+u = a^u$ (the image of
$a$ by $u$) and prove that the two axioms of affine spaces hold. For
all $v\in \restr{G}{O}$ it is clear that
\[(a+u)+v = (a^u)^v = a^{(u+v)} = a+(u+v).\]
Consider $\affbij{a}:\restr{G}{O} \rightarrow O$ such that
$\affbijof{a}{u} = a+u$, we prove that $\affbij{a}$ is bijective. It
is obviously onto since $\restr{G}{O}$ is transitive on $O$: $\forall
b\in O$, $\exists u\in \restr{G}{O}$ such that $b = a^u =
\affbijof{a}{u}$. Assume now that $\affbijof{a}{u} = \affbijof{a}{v}$,
i.e., $a^u=a^v$, then for any $b\in O$, if $w\in \restr{G}{O}$ is such
that $a^w=b$, then
\[b^u = a^{(w+u)} = a^{(u+w)} = (a^u)^w = (a^v)^w = b^v,\]
hence $u=v$, and $\affbij{a}$ is injective. 
\end{proof}

This is nothing more than a geometric interpretation of a known
result: that transitive Abelian groups are ``regular'' (see
\cite[theorem 10.3.4]{scott}). This entails that $\card{O} =
\card{\restr{G}{O}}$, and that the external sum is \emph{regular},
i.e., an equality $a+u=b$ determines every term from the two
others. In particular the unique vector $u\in \restr{G}{O}$ such that
$a+u=b$ is usually written $u=b-a$ (or $\overrightarrow{ab}$), and we
also write $b=a-u$ for $b=a+(-u)$, which is equivalent to
$b+u=a$. Note that there is one external sum (and difference) per
orbit, but since they are disjoint it is unambiguous to denote them
with the same symbol.

In the sequel we use the additive notation on sets $S,S'$ of vectors
or elements of an orbit, i.e., if $\varepsilon\in\set{+,-}$ then
$S\varepsilon S' = \setof{s\varepsilon s'}{s\in S, s'\in S'}$. If one
set is a singleton, say $S=\set{s}$, we write $s\varepsilon S'$ for
$\set{s}\varepsilon S'$. This is of course compatible with the
notations already used when $S$ and $S'$ are subspaces of $\supspc$.

\subsection{Computing a basis for $\supspc$}\label{sec-compbasis}

Since $\restr{G}{O}$ is a finite $\eF{p}$-vector space, its
cardinality must be $p^{d_O}$. But this cardinality is also that of
$O$ which is known, hence the dimension $d_O$ can be computed:
$d_O=\log_p\card{O}$. Furthermore, since $\restr{g_1}{O}, \ldots,
\restr{g_m}{O}$ is a generating set for $\restr{G}{O}$, it is
possible to extract a basis $\supbasis_O$ of $\restr{G}{O}$ from
this set by discarding $m-d_O$ linearly dependent vectors. For this a test
of linear dependence is required.

\begin{lemma}\label{lem-orblin}
 For any subspace $H$ of $\restr{G}{O}$, any
  $u\in\restr{G}{O}$ and any $a\in O$,
\[u\in H\ \mbox{iff}\ a^u\in a^H.\]
\end{lemma}
\begin{proof}
By regularity of the external sum, $u\in H$ iff $a+u\in a+H$. By
definition $a+u = a^u$ and $a+H = \setof{a+v}{v\in H} =
\setof{a^v}{v\in H} = a^H$. 
\end{proof}

Linear dependence is therefore reduced to computing the orbit of an
arbitrary point $a\in O$. A basis of $\restr{G}{O}$ can be built step by
step by computing the corresponding orbit partition $\QP$ of $O$, with the
following function $\buildbasisfn$:

\[\left\{\begin{array}{lcll}
\buildbasis{[\ ]}{\mathbf{h}}{\QP} & = & \mathbf{h}, & \\
\buildbasis{[u]@l}{\mathbf{h}}{\QP} & = & \buildbasis{l}{\mathbf{h}}{\QP}\ 
& \mbox{if}\ a^u\in a[\QP],\\
& = & \buildbasis{l}{[u] @ \mathbf{h}}{\orbp{u}\merge\QP}\ &\mbox{otherwise.}
\end{array}\right.\]

Here $l$ is a list of vectors, $@$ is the concatenation of lists and $[\ ]$
is the empty list. 
\begin{lemma}
$\buildbasis{[\restr{g_1}{O},\ldots, \restr{g_m}{O}]}{[\ ]}{\botp{O}}$
  is a basis of $\restr{G}{O}$.
\end{lemma}
\begin{proof}
Let $\supbasis_O = \buildbasis{[\restr{g_1}{O},\ldots,
    \restr{g_m}{O}]}{[\ ]}{\botp{O}}$. Since $\botp{O}$ is the orbit
partition of the trivial group $\set{0}$ on $O$, the invariant $\QP =
\orbp{\gen{\mathbf{h}}}$ is maintained throughout the
computation. This means that the class $a[\QP]$ of $a$ modulo $\QP$ is
the orbit $a^{\gen{\mathbf{h}}}$, and that $a^u\in a[\QP]$ is
equivalent to $u\in\gen{\mathbf{h}}$ by Lemma \ref{lem-orblin}. Hence
$u$ is added to $\mathbf{h}$ if and only if it is linearly independent
from the latter, which proves that $\mathbf{h}$ remains free. This
also proves that $\gen{l@\mathbf{h}}$ is invariant, hence
$\gen{\supbasis_O} = \gen{\restr{g_1}{O},\ldots, \restr{g_m}{O}} =
\restr{G}{O}$ and $\supbasis_O$ is free, it is thus a basis of
$\restr{G}{O}$.
\end{proof}

\begin{example} Let $O=\set{1,\ldots,8}$ and
\begin{eqnarray*}
g_1 & = & (1\ 2)(3\ 4)(5\ 6)(7\ 8),\\
g_2 & = & (1\ 5)(2\ 6)(3\ 7)(4\ 8),\\
g_3 & = & (1\ 3)(2\ 4)(5\ 7)(6\ 8),
\end{eqnarray*}
We choose $a=1$, then
\begin{eqnarray*}
\supbasis_O & = & \buildbasis{[g_1,g_2,g_3]}{[\ ]}{\botp{O}}  \\
& = & \buildbasis{[g_2,g_3]}{[g_1]}{\orbp{g_1}}\  \mbox{since}\ 1^{g_1}
    = 2 \not\in 1[\botp{O}] = \set{1}\\
& = & \buildbasis{[g_3]}{[g_2,g_1]}{\orbp{g_2}\merge\orbp{g_1}}\ 
    \mbox{since}\ 1^{g_2} = 5 \not\in 1[\orbp{g_1}] = \set{1,2}\\
& = & \buildbasis{[\ ]}{[g_3,g_2,g_1]}{\topp{O}}\ 
    \mbox{since}\ 1^{g_3} = 3 \not\in 1[\orbp{g_2}\merge\orbp{g_1}] =
    \set{1,2,5,6}\\
& = & [g_3,g_2,g_1].
\end{eqnarray*}
\end{example}

Of course the algorithm can be interrupted once $\mathbf{h}$ has $d_O$
elements (or equivalently when $\QP = \topp{O}$). Building
$\supbasis_O$ requires at most $m$ recursive calls and the computation
of exactly $d_O\leq m$ orbit partitions of $O$, each being polynomial
in $\card{O}\leq n$, hence $\supbasis_O$ is computed in time
polynomial in $n$ and $m$.

The bases $\supbasis_O$ can be concatenated (in an arbitrary order)
to form a basis $\supbasis$ of $\supspc$. The length $d$ of this basis may
be greater than $m$, but since \[d_O\ =\ \log_p\card{O}\ \leq
\ \frac{\card{O}}{p},\] necessarily \[d\ =\ \sum_{O\in {\orbpG}}d_O\ \leq
\ \sum_{O\in {\orbpG}}\frac{\card{O}}{p}\ =\ \frac{n}{p},\] hence computing
$\supbasis$  is again polynomial in $n$ and $m$.

\subsection{Computing coordinates in the basis for $\supspc$}\label{sec-compcoord}

The coordinates of any vector $u\in \supspc$ in basis $\supbasis$ can be
obtained by computing, for all $O\in {\orbpG}$, the coordinates
of $\restr{u}{O}\in \restr{G}{O}$ in the basis $\supbasis_O$, and by
concatenating these coordinates in the same order as the one used to
build $\supbasis$. We show how to compute the coordinates in
$\supbasis_O$ of the permutations in $\restr{G}{O}$.

Since $\supbasis_O$ is a basis of $\restr{G}{O}$, there is a 1-1
correspondence from the tuples $\tuple{x_1,\ldots,x_{d_O}}\in
(\eF{p})^{d_O}$ to the elements of $\restr{G}{O}$, given by
$\sum_{i=1}^{d_O} x_ih_i$, which can be computed in polynomial time:
each $x_ih_i$ requires composing $x_i-1< p$ times the permutation
$h_i$ of $O$ with itself, hence computing a permutation from its coordinates
in $\supbasis_O$ can be computed in time $\mathrm{O}(d_Op\card{O})$, which is
bounded by $\mathrm{O}(n\log n)$ (since $p$ is a constant).

Computing the coordinates in $\supbasis_O$ of a given $u\in
\restr{G}{O}$ means computing the inverse of the previous
correspondence. This can be performed by browsing through all possible
values $\tuple{x_1,\ldots, x_{d_O}}\in(\eF{p})^{d_O}$ until
$\sum_{i=1}^{d_O} x_ih_i$ equals $u$. Since $\card{O} =
\card{\restr{G}{O}} = p^{d_O}$, this requires at most $\card{O}$
iterations, hence can be computed in time $\mathrm{O}(d_Op\card{O}^2)$, which
is bounded by $\mathrm{O}(n^2\log n)$.

The same technique may be applied if $u$ is given as $b-a$, where
$a,b\in O$. In this case it is necessary to check whether $b = a +
\sum_{i=1}^{d_O} x_ih_i$, i.e., whether $b$ is the image of $a$ by the
permutation $\sum_{i=1}^{d_O} x_ih_i$. The complexity is thus the same as
above. Note that this also allows to compute $b-a$ explicitly as
a permutation.

From a practical point of view we should avoid repeated computations
of permutations from coordinates: this could be done by storing values
in a suitable array. Using our geometric interpretation, we choose
arbitrarily an origin $a\in O$ and define the coordinates of any point
$b\in O$ as those of the vector $b-a$ relative to $\supbasis_O$. We
can therefore fill an array that associates its coordinates to each
entry $b\in O$, by browsing through all possible coordinates as
explained above. Since this array has $\card{O}$ entries, filling it
takes polynomial time. Then, given a permutation $u$, we need only
compute the image $b=u^a$ and pick the coordinates of $b$ in the
array; these are the coordinates of $u$.

In the sequel we write $\supbasis=\basevec{1},\ldots,\basevec{d}$, and
for any $u\in \supspc$, if $u=\sum_{i=1}^d x_i\basevec{i}$ where the
$x_i\in\eF{p}$ are the coordinates of $u$ in $\supbasis$, we write
$u_{\supbasis} = \transpose{(x_1\ \cdots\ x_d)}$ the column matrix of
these coordinates.

\subsection{A characterization of linear varieties in $\supspc$}\label{sec-linvar}

Since the super-space $\supspc$ is isomorphic to the vector space
$(\eF{p})^d$, and this isomorphism can be computed (through the
coordinates in $\supbasis$) in both directions, computations with
matrices can be substituted for computations with permutations. This
means that standard algorithms from linear algebra apply, in
particular Gaussian elimination. Note that \emph{exact} computations can be
performed in $\eF{p}$, including division, in time at most quadratic
in the number of bits (see, e.g. \cite[p.117]{Geddes92}), hence in constant
time in the present context.

In particular, it is now straightforward to test whether a family of
$l\leq d$ vectors $u_1,\ldots, u_l \in \supspc$ is linearly dependent,
by first computing the coordinates $u_i=\sum_{j=1}^d
x_{ij}\basevec{j}$ and then by performing Gaussian elimination on the
$l\times d$-matrix $(x_{ij})$ (which requires a number of operations
at most cubic in $d$). The family is linearly dependent iff Gaussian
elimination yields a zero row in the resulting matrix (the number of
non-zero lines after Gaussian elimination is the rank of the matrix,
i.e., the dimension of the space $\gen{u_1,\ldots, u_l}$).

Assume we are given a \emph{linear variety} $v+H$ of $\supspc$ by the
permutation $v$ and a generating set for the subspace $H$. We can compute
the coordinates in $\supbasis$ of these permutations and, using the
previous procedure, extract a basis $h_1,\ldots,h_{d'}$ of $H$ from the
generators of $H$, where $d'$ is the dimension of $H$. The vectors
$h_1,\ldots,h_{d'}$ together with the vectors in $\supbasis$ form a
generating family of $\supspc$; the free family $h_1,\ldots,h_{d'}$ can
therefore be completed into a basis $\mathbf{h}=h_1,\ldots,h_d$ of
$\supspc$ by adding $d-d'$ vectors taken from $\supbasis$. If $P$ denotes
the matrix whose $i^{\mathrm{th}}$ column is $\colmat{(h_i)}{\supbasis}$
then $P$ is the change of basis matrix from $\mathbf{h}$ to $\supbasis$:
$\colmat{u}{\supbasis} = P\colmat{u}{\mathbf{h}}$ for all
$u\in\supspc$. This matrix is invertible and its inverse can be computed
using the Gauss-Jordan algorithm in time cubic in $d$.

This means that the coordinates of any vector $u$ in $\mathbf{h}$ (or
in any basis) can be computed in polynomial time through
$\colmat{u}{\mathbf{h}} = \inv{P} \colmat{u}{\supbasis}$. Membership
of $u$ to the subspace $H$ may be checked simply by making sure the
last $d-d'$ coordinates of $\colmat{u}{\mathbf{h}}$ are equal to
zero. This can be expressed by building the diagonal $d\times
d$-matrix $D$ with ones on the last $d-d'$ positions of the diagonal
and zeroes elsewhere:
\[D = \left(\begin{array}{cc}
\mathbf{0} & \mathbf{0} \\
\mathbf{0} & I \\
\end{array}\right),\]
where $I$ is the $(d-d')\times(d-d')$ identity matrix. Thus vector
$u\in \supspc$ belongs to $H$ iff $D\colmat{u}{\mathbf{h}} = 0$. Let
$\matrvar{H} = D\inv{P}$, we have shown that:

\begin{lemma}\label{lem-matrvar}
From any linear variety $v+H$ of $\supspc$ can be computed in
polynomial time a $d\times d$-matrix $\matrvar{H}$ such that $\forall
u\in \supspc$, \[u \in
v+H\ \mbox{iff}\ \matrvar{H}\colmat{u}{\supbasis} =
\matrvar{H}\colmat{v}{\supbasis}.\]
\end{lemma}

Conversely, it is well known that the set of solutions $u$ of any
system of linear equations on $d$ unknowns (the coordinates of $u$ in
$\supbasis$) is either empty or a linear variety of $\supspc$.

\section{Solving linear constraints}\label{sec-solvlincstr}

The vector space $G$ being a linear variety of $\supspc$ by Lemma
\ref{lem-E}, its elements are characterized as the solutions $u$ of a
system of linear equations $\matrvar{G}\colmat{u}{\supbasis} = 0$, as shown
in Lemma \ref{lem-matrvar} (by taking $v=0$). We now investigate how to
characterize the constraint $\cstr$ by another system of linear equations.

\subsection{Constraints as sets of vectors}

The first step towards this characterization is to explicitly
represent the set of vectors $u\in E$ that satisfy a constraint
$\cstr$.
Let $\displaystyle V_O = \bigcap_{a\in O}\cstrat{a}-a$ for all
$O\in\orbpG$.

\begin{lemma}\label{lem-VO}
A vector $u\in \supspc$ satisfies $\cstr$ iff
$\displaystyle u\in\sum_{O\in\orbpG}V_O$. 
\end{lemma}
\begin{proof}
Assume a vector $u\in \supspc$ satisfies the constraint $\cstr$, i.e., for
all $a\in A$, $a^u\in\cstrat{a}$. Then for all $O$ in the orbit
partition ${\orbpG}$ and for all $a\in O$, since
$a^u=a^{\restr{u}{O}}\in O$ and $(O,\restr{G}{O})$ is an affine space,
we may write $a+\restr{u}{O}\in\cstrat{a}$, and since
$\cstrat{a}\subseteq O$, this is equivalent to
$\restr{u}{O}\in\cstrat{a}-a$. But this is true for all $a\in O$,
hence $\restr{u}{O}\in V_O$. Since $u=\sum_{O\in
  {\orbpG}}\restr{u}{O}$, vector $u$ must belong to $\sum_{O\in
  {\orbpG}}V_O$.

Conversely, let $u\in \sum_{O\in {\orbpG}}V_O$ and let $a\in
A$. If $O=a^G$ then $\restr{u}{O}\in V_O$, hence
in particular $\restr{u}{O}\in \cstrat{a}-a$, i.e., $a^u =
a+\restr{u}{O} \in \cstrat{a}$. Hence $u$ satisfies the
constraint $\cstr$. 
\end{proof}

The sets $V_O$ can be computed for each $O\in {\orbpG}$ by selecting
an arbitrary $a\in O$ and computing the coordinates of $c-a$ in
$\supbasis_O$ for all $c\in\cstrat{a}$, and this operation is
polynomial in $n$ for each $c$ as mentioned in section \ref{sec-compcoord}.
Provided $\cstr$ is a $k$-constraint this yields at most $k$ elements
in $\cstrat{a}-a$. Then, $V_O$ is the set of those vectors $u$ from
$\cstrat{a}-a$ that also belong to $\cstrat{b}-b$ for all $b\in
O\setminus\set{a}$, i.e., such that $b^u\in\cstrat{b}$, which requires
that $u$ be also computed as an explicit permutation.  Computing the
coordinates of all the elements of $V_O$ is therefore polynomial in $n$
and $k$.

\subsection{Linear constraints}\label{sec-lincstr}

The size of the set $\sum_{O\in {\orbpG}}V_O$ is $\prod_{O\in
  {\orbpG}}\card{V_O}$.  If one of the sets $V_O$ is empty, then
obviously constraint $\cstr$ is unsatisfiable. But in general this set
is exponential in the number of $G$-orbits (and therefore in $n$, the
worst case being $2^{\frac{n}{2}}$ with $\frac{n}{2}$ orbits of size
2). This motivates the following definition.

\begin{definition}
$\cstr$ is a \emph{linear constraint} if $\displaystyle \sum_{O\in
    {\orbpG}}V_O$ is either empty or a linear variety of $\supspc$.
\end{definition}

In order to check whether a constraint $\cstr$ is linear or not, we
provide a simple characterization of this property. For $O\in\orbpG$,
let $w_O$ be an arbitrary element of $V_O$ and $\cstrspc_O=V_O-w_O$.

\begin{lemma}\label{lem-dimvar}
$\displaystyle \sum_{O\in\orbpG}V_O$ is a linear variety of $\supspc$ iff
  $\forall O\in\orbpG$, $\dim\gen{\cstrspc_O} = \log_p\card{V_O}$.
\end{lemma}
\begin{proof}
Since $\supspc=\bigoplus_{O\in\orbpG}\restr{G}{O}$ and
$V_O\subseteq\restr{G}{O}$, it is clear that $\sum_{O\in\orbpG}V_O$
is a linear variety of $\supspc$ iff $V_O$ is a linear variety of
$\restr{G}{O}$ for all $O\in\orbpG$. If this is so then $\cstrspc_O$
is a subspace of $\restr{G}{O}$ that does not depend on $w_O$. Hence
this is equivalent to all the $\cstrspc_O$'s being subspaces, hence to
$\gen{\cstrspc_O}=\cstrspc_O$. Again this is equivalent
to $p^{\dim\gen{\cstrspc_O}} = \card{\gen{\cstrspc_O}} =
\card{\cstrspc_O} = \card{V_O}$. 
\end{proof}

The dimension of $\gen{\cstrspc_O}$ is the rank of the matrix formed
by the coordinates of the vectors in $\cstrspc_O$, which can be
computed in polynomial time and compared to
$\log_p\card{V_O}$. Computing this rank is of course useless if
$\log_p\card{V_O}$ is not an integer. Hence the linearity of $\cstr$
can be tested in time polynomial in $n$ (since $\card{\orbpG}\leq n$).
Note that the case where $V_O$
is a singleton corresponds to $\cstrspc_O$ being reduced to the
trivial subspace $\set{0}$ of dimension 0.

Let $w=\sum_{O\in {\orbpG}}w_O$ and $\cstrspc=\bigoplus_{O\in
  {\orbpG}}\cstrspc_O$ for every $O\in {\orbpG}$, so that $\sum_{O\in
  {\orbpG}}V_O$ is the linear variety $w+\cstrspc$ (assuming it is not
empty).

\begin{theorem}\label{thm-lin}
  If $\cstr$ is linear then the problem \textsc{gc} is equivalent to a
  system of linear equations on coordinates of a solution in
  $\supbasis$, and this system can be computed and solved in
  polynomial time.
\end{theorem}
\begin{proof}
  If $\sum_{O\in {\orbpG}}V_O = \emptyset$ then the instance
  $\tuple{A,G,\cstr}$ of \textsc{gc} has no solution, which is
  equivalent to the linear equation $0=1$. Otherwise $\sum_{O\in
    {\orbpG}}V_O = w+\cstrspc$ and by Lemma \ref{lem-VO} any $u\in
  \supspc$ is a solution of the instance $\tuple{A,G,\cstr}$ iff $u\in
  G$ and $u\in w+H$, which is equivalent by Lemma \ref{lem-matrvar} to
  \[\left\{\begin{array}{l}\matrvar{G}\colmat{u}{\supbasis}=0\\
  \matrvar{\cstrspc}\colmat{u}{\supbasis}=
  \matrvar{\cstrspc}\colmat{w}{\supbasis}.\end{array}\right.\] This is
  a system of $2d$ linear equations on $d$ unknowns, it can be solved
  by Gaussian elimination in time cubic in $d$. 
\end{proof}

\begin{corollary} The problem \textsc{gc} restricted to $p=k=2$ is polynomial.
\end{corollary}
\begin{proof}
Every non-empty $V_O$ has at most $k=2$ elements. If $V_O = \set{u} =
u+\set{0}$ then $\dim\gen{\cstrspc_O} = \dim\set{0} = 0 =
\log_2\card{V_O}$; if $V_O = \set{u,v} = u+\set{0,v-u}$ then
$\dim\gen{\cstrspc_O} = \dim\gen{v-u} = 1 = \log_2 \card{V_O}$, hence
according to Lemma \ref{lem-dimvar} the constraint $\cstr$ is
linear. 
\end{proof}

This result holds both for the decision and the search problem.

\section{NP-Completeness results}\label{sec-NPC}

\subsection{Constraints with more than 2 elements}\label{sec-bigk}

If one $V_O$ has more than 2 elements then $\cstr$ may not be a linear
constraint, and therefore the previous polynomial algorithm may
not apply. In fact, we now prove that allowing constraints of
a cardinality greater than 2 makes the decision problem \textsc{gc}
NP-hard, whatever the value of $p$. We proceed by reduction from the
problem of 1-satisfiability of positive $k$-clauses.

Let $\Sigma$ be a finite set of \emph{propositional
  variables} (which will be denoted by Greek letters), a
\emph{positive $k$-clause} is a subset $C\subseteq \Sigma$ of
cardinality $k$. Let $S$ be a finite set of such clauses, then $S$ is
\emph{1-satisfiable} if there is an \emph{interpretation} $I\subseteq
\Sigma$ such that every clause $C\in S$ contains exactly one element
in $I$.

Given $\Sigma$ and $S$, we build an instance of the decision problem
\textsc{gc} (restricted to $\eF{p}$-vector spaces) whose
satisfiability is equivalent to the 1-satisfiability of $S$. Furthermore,
the construction is polynomial in the size of the input clause
set, i.e., it is polynomial in $\card{\Sigma}$ and $\card{S}$ (not
necessarily in $k$, which is a constant). This transformation
consists in interpreting propositional variables and clauses
in $\eF{p}$.

We consider the space of functions from $\Sigma$ to $\eF{p}$, written
$\eF{p}^\Sigma$, with the standard sum $(u+v)(\alpha) =
u(\alpha)+v(\alpha)$ and the external product $(xu)(\alpha) =
xu(\alpha)$ for all $\alpha\in \Sigma$, $u,v\in\eF{p}^{\Sigma}$ and
$x\in\eF{p}$. This is an $\eF{p}$-vector space of dimension
$\card{\Sigma}$, with generating set $\setof{\kron{\beta}}{\beta\in
  \Sigma}$, where $\kron{\beta}(\alpha) = 1$ if $\alpha=\beta$ and 0
otherwise.

We construct a permutation group as a homomorphic
image of the elementary Abelian $p$-group $\eF{p}^{\Sigma}$. The
elements to be permuted are those of the set \[A_S = \biguplus_{C\in
  S}\eF{p}^C\] of functions from $C$ to $\eF{p}$, where $C$ is a
positive $k$-clause belonging to $S$. The
cardinality of $A_S$ is $\sum_{C\in S}p^{\card{C}} = p^k\card{S}$. Each
element of $A_S$ is a function that associates integers modulo $p$ to
$k$ propositional variables, hence it can be encoded in constant
size.

\begin{example}
  We consider the set $\Sigma=\set{\alpha,\beta,\gamma}$ and assume that
  $p=2$. There are exactly 8 functions from $\Sigma$ to $\eF{2}$,
  which we denote according to the following scheme.
  \[\begin{array}{ccccccccc}
    \ \Sigma\ &\ g_0\ &\ g_1\ &\ g_2\ &\ g_3\ &\ g_4\ &\ g_5\ &\ g_6\ &\ g_7\ \\
    \alpha & 0 & 0 & 0 & 0 & 1 & 1 & 1 & 1\\
    \beta & 0 & 0 & 1 & 1 & 0 & 0 & 1 & 1 \\
    \gamma & 0 & 1 & 0 & 1 & 0 & 1 & 0 & 1
  \end{array}\]
  Let $C=\set{\alpha,\beta,\gamma}$, which is a positive 3-clause on $\Sigma$, and
  $S=\set{C}$. Since $C=\Sigma$, we have \[A_S\ =\ \biguplus_{C'\in
    S}\eF{2}^{C'}\ =\ \eF{2}^{C}\ =\ \eF{2}^{\Sigma}\ =\ \set{g_0,\ldots,\, g_7}.\]
\end{example}

Let $f$ denote the function from $\eF{p}^{\Sigma}$ to $\Sym{A_S}$
defined for all $u\in \eF{p}^{\Sigma}$, $C\in S$ and $w\in \eF{p}^C$,
by $w^{f(u)} = w+\restr{u}{C}$. It is straightforward to verify that
$f(u)$ is indeed a permutation of $A_S$, and that $f(u)f(v) = f(u+v)$;
hence $f$ is a group morphism and the group $G_S$ generated by the
permutations $\setof{f(\kron{\alpha})}{\alpha\in \Sigma}$ is an
elementary Abelian $p$-group. This generating set can obviously be
computed in time polynomial in $\card{\Sigma}$ and $\card{A_S}$.

\begin{example}
  We compute $f(g_3)$. For all $w\in\eF{2}^{\Sigma}$, $w^{f(g_3)} = w
  + g_3$. We have $g_0 + g_3 = g_3$, $g_1 + g_3 = g_2$, etc. and we
  easily obtain, in cycle notation
    \[f(g_3)  \ =\  (g_0\ g_3)(g_1\ g_2)(g_4\ g_7)(g_5\ g_6).\]
\end{example}

Finally, let $\cstrs{S}$ be the $k$-constraint on $G_S$ defined
for all $C\in S$ and $w\in \eF{p}^C$, by \[\cstrsat{S}{w} =
\setof{w+\restr{(\kron{\alpha})}{C}}{\alpha\in C}.\]

\begin{example}
  Since $\kron{C}(\alpha) = g_4$, $\kron{C}(\beta) = g_2$ and $\kron{C}(\gamma) =
  g_1$, we have for all $w\in\eF{2}^C$, $\cstrsat{S}{w} =
  \set{w+g_4,\, w+g_2,\, w+g_1}$. This yields for instance
  $\cstrsat{S}{g_3} = \set{g_7,\, g_1,\, g_2}$.
\end{example}

\begin{lemma}
$S$ is 1-satisfiable iff $\cstrs{S}$ is satisfiable in $G_S$.
\end{lemma}
\begin{proof}
First assume that $S$ is 1-satisfiable, i.e., there is an
interpretation $I\subseteq \Sigma$ such that every clause $C\in S$ has
exactly one element in $I$. Let $u\in\eF{p}^{\Sigma}$ be defined by
$u(\alpha) = 1$ if $\alpha\in I$ and 0 otherwise. For $C\in S$, if
$\set{\alpha} = C\cap I$ then it is clear that $\restr{u}{C} =
\restr{(\kron{\alpha})}{C}$, so that for all $w\in \eF{p}^C$,
$w^{f(u)} = w + \restr{u}{C} = w + \restr{(\kron{\alpha})}{C} \in
\cstrsat{S}{w}$. Thus $f(u)\in G_S$ satisfies $\cstrs{S}$.

Conversely, suppose there is an element $f(u)$ of $G_S$ that satisfies
$\cstrs{S}$, let $I=\setof{\alpha\in \Sigma}{u(\alpha)=1}$ and let $C$
be any clause in $S$. For all $w\in \eF{p}^C$, since
$w^{f(u)}\in\cstrsat{S}{w}$ there is an $\alpha\in C$ such that
$w^{f(u)} = w+\restr{(\kron{\alpha})}{C}$, hence such that
$\restr{u}{C} = \restr{(\kron{\alpha})}{C}$. Necessarily $u(\alpha) =
1$ and therefore $\alpha\in C\cap I$. If $\beta\in C\cap I$ we
similarly obtain $\restr{u}{C} = \restr{(\kron{\beta})}{C}$, hence
$\beta=\alpha$. This proves that $C\cap I$ is a singleton for all
$C\in S$, and that $I$ 1-satisfies $S$. 
\end{proof}
\begin{theorem}
For any prime $p$, the problem of solving $k$-constraints in
$\eF{p}$-vector spaces is NP-complete if $k\geq 3$.
\end{theorem}
This follows from the NP-completeness of the problem of determining
1-satisfiability of a set of positive 3-clauses (see \cite[problem
  L04, p. 259]{Garey&Johnson79}).

\subsection{Constraints with at most 2 elements}\label{sec-bigp}

If a $V_O$ has exactly two elements but $p\geq 3$, then $V_O$ cannot
be a linear variety of $\supspc$ and the algorithm of Section
\ref{sec-lincstr} necessarily fails. We prove that
allowing $p\geq 3$ makes the restriction of the decision problem
\textsc{gc} to $\eF{p}$-vector spaces NP-complete, even
with constraints of cardinality at most 2. We proceed by reduction
from 1-satisfiability of positive $p$-clauses.

Given a set $S$ of positive $p$-clauses on $\Sigma$, we again consider
the $\eF{p}$-vector space $\eF{p}^{\Sigma}$ and define the set
\[A'_S\ =\ \Sigma\times \eF{p}\ \uplus\ S\times\eF{p},\]
whose cardinality is $p\card{\Sigma} + p\card{S}$. Let $f'$ be the
function from $\eF{p}^{\Sigma}$ to $\Sym{A'_S}$ defined for all
$u\in \eF{p}^{\Sigma}$, $\tuple{\alpha,x}\in \Sigma\times \eF{p}$ and
$\tuple{C,y}\in S\times\eF{p}$, by
\begin{eqnarray*}
\tuple{\alpha,x}^{f'(u)} & = & \tuple{\alpha,x+u(\alpha)},\\
\tuple{C,y}^{f'(u)} & = & \tuple{C,y + \sum_{\beta\in C}u(\beta)}.
\end{eqnarray*}
It is straightforward to verify that $f'(u)$ is a permutation of
$A'_S$: if $\tuple{\alpha,x}^{f'(u)} = \tuple{\alpha',x'}^{f'(u)}$
then $\alpha = \alpha'$ and then $x=x'$; if $\tuple{C,y}^{f'(u)} =
\tuple{C',y'}^{f'(u)}$ then $C=C'$ and then $y=y'$. It is obvious that
$f'(u)f'(v) = f'(u+v)$, hence $f'$ is a group morphism and the group
$G'_S$ generated by the permutations
$\setof{f'(\kron{\alpha})}{\alpha\in \Sigma}$ is therefore an
elementary Abelian $p$-group. This generating set can be computed in
time polynomial in $\card{\Sigma}$ and $\card{A'_S}$.

\begin{example}
  We assume the same $\Sigma$, $p$, $C$ and $S$ as in the running
  example of Section \ref{sec-bigk}. Then
  \[A'_S\ =\ \set{\tuple{\alpha,0},\,\tuple{\alpha,1},\,
    \tuple{\beta,0},\,\tuple{\beta,1},\,\tuple{\gamma,0},\,\tuple{\gamma,1},\, 
    \tuple{C,0},\, \tuple{C,1}}.\]
  The permutations $f'(u)$ for $u\in\eF{2}^{\Sigma}$ can again be expressed in
  cycle notation, for instance:
  \begin{eqnarray*}
    f'(g_2) & \ =\ & (\tuple{\beta,0}\ \tuple{\beta,1})
    (\tuple{C,0}\ \tuple{C,1}), \\
    f'(g_3) & \ =\ & (\tuple{\beta,0}\ \tuple{\beta,1})
    (\tuple{\gamma,0}\ \tuple{\gamma,1}).
  \end{eqnarray*}
  $\tuple{C,0}$ is a fix-point of $f'(g_3)$ since
  $g_3(\alpha) + g_3(\beta) + g_3(\gamma) = 0 + 1 + 1 = 0$.
\end{example}

Let $\cstrps{S}$ be the $2$-constraint on $G'_S$ defined, for all
$\tuple{\alpha,x}\in \Sigma\times \eF{p}$ and $\tuple{C,y}\in
S\times\eF{p}$, by
\begin{eqnarray*}
\cstrpsat{S}{\tuple{\alpha,x}} & = &
\set{\tuple{\alpha,x},\tuple{\alpha,x+1}},\\
\cstrpsat{S}{\tuple{C,y}} & = & \set{\tuple{C,y + 1}}.
\end{eqnarray*}

\begin{example}
 Obviously $\cstrpsat{S}{\tuple{\alpha,0}} =
 \cstrpsat{S}{\tuple{\alpha,1}} =
 \set{\tuple{\alpha,0},\tuple{\alpha,1}}$, and similarly for $\beta$
 and $\gamma$. For the other two elements of $A'_{S}$ we have:
  \[\cstrpsat{S}{\tuple{C,0}}\ =\ \set{\tuple{C,1}}\ \ \mbox{and}\ \ 
  \cstrpsat{S}{\tuple{C,1}}\ =\ \set{\tuple{C,0}}.\] 
\end{example}

\begin{lemma}
$S$ is 1-satisfiable iff $\cstrps{S}$ is satisfiable in $G'_S$.
\end{lemma}
\begin{proof}
Assume $S$ is 1-satisfiable, let $I$ be an interpretation of $S$ and
consider $u\in\eF{p}^{\Sigma}$ defined by $u(\alpha) = 1$ if
$\alpha\in I$ and 0 otherwise. The constraint on any
$\tuple{\alpha,x}\in \Sigma\times \eF{p}$ is satisfied since
$\tuple{\alpha,x}^{f'(u)} = \tuple{\alpha,x+u(\alpha)} \in
\cstrpsat{S}{\tuple{\alpha,x}}$. For all $\tuple{C,y}\in
S\times\eF{p}$ there is a unique $\beta\in C$ such that $u(\beta) = 1$
and $\restr{u}{C}$ is zero elsewhere, hence $\tuple{C,y}^{f'(u)} =
\tuple{C,y + \sum_{\beta\in C}u(\beta)} = \tuple{C, y+1}\in
\cstrpsat{S}{\tuple{C,y}}$. This shows that $f'(u)\in G'_S$ satisfies
$\cstrps{S}$.

Conversely, suppose that an element $f'(u)$ of $G'_S$ satisfies
$\cstrps{S}$ and let $I=\setof{\alpha\in \Sigma}{u(\alpha)=1}$. Then
$u(\alpha)\neq 1$ for all $\alpha \in \Sigma\setminus I$, and
$u(\alpha)\in\set{0,1}$ since $\tuple{\alpha,x}^{f'(u)} \in
\cstrpsat{S}{\tuple{\alpha,x}}$, thus $u(\alpha)=0$
(modulo $p$). Let $C$ be a clause in $S$, the constraint yields
$\tuple{C,0}^{f'(u)} = \tuple{C,1}$, hence $\sum_{\beta\in C}u(\beta)
= 1$. The terms of this sum are either 0 or 1, hence at least one must
be a 1. Furthermore, there are at most $p$ terms, hence only one can
be a 1, say $u(\alpha)=1$, then by definition of $u$, $\alpha$ is the
only member of $C$ that belongs to $I$. Hence $I$ 1-satisfies $S$. 
\end{proof}

\begin{theorem}
For any prime $p\geq 3$, the problem of solving 2-constraints in
$\eF{p}$-vector spaces is NP-complete.
\end{theorem}

\section{Experimental results}\label{sec-expe}

The polynomial algorithm for solving linear constraints has been
implemented in the GAP system, using its facilities on permutations,
matrix algebra and finite fields. The implementation, nicknamed
\texttt{Solvect} (see ``downloads'' page on \texttt{capp.imag.fr}),
 is straightforward except for the fact that
coordinates in the transitive constituents are kept in memory and
hence computed only once (this is performed while computing the orbit
partition of $G$). Its performance has been measured against a
general purpose group search algorithm provided in GAP, described in
\cite{Leon91} and refined in \cite{Theissen97}. The call to this
algorithm is
\[\mathtt{ElementProperty}(\gen{g_1,\ldots, g_m},\ g \mapsto
\forall a \in A, a^g\in \cstrat{a}\ );\] which returns an element of $G=
\gen{g_1,\ldots, g_m}$ satisfying the specified property if there is
one, and \texttt{fail} otherwise.

\begin{table}[t]
\[{ \begin{array}{|c||r|r|r|r|}
\hline
\hfill\dim G \hfill & \hfill n \hfill & \hfill d \hfill & \hfill t_1 \hfill & \hfill t_2 \hfill \\
\hline
5 & 26.4\pm 56\%\ \,& 6.9\pm 29\% &\ 0.384\pm 310\% & 0.82 \pm 214\% \\
10 & 270\pm 150\% &\ 14.9\pm 26\% & 1.46\pm 137\% & 20.3 \pm 56\%\ \,\\
15 & 785\pm 277\% & 27.2\pm 28\% & 5.4\pm 143\% & 631 \pm 106\% \\
20 & 1\,060\pm 229\% & 35.4\pm 28\% & 10.1\pm 100\% &\ 19\,900 \pm 90\%\ \,\\
25 & 2\,230\pm 175\% & 52.2\pm 31\% & 25.8\pm 73\%\ \,& \hfill -\hfill  \\
30 & 2\,730\pm 148\% & 67.4\pm 34\% & 45.3\pm 66\%\ \,& \hfill - \hfill \\
35 & 2\,870\pm 107\% & 79\pm 35\% & 65.3\pm 67\%\ \,& \hfill - \hfill \\
40 & 8\,510\pm 94\%\ \,& 94.2\pm 38\% & 147\pm 69\%\ \,& \hfill - \hfill \\
45 & 7\,680\pm 77\%\ \,& 125\pm 37\% & 241\pm 64\%\ \,& \hfill - \hfill \\
50 &\ 12\,800\pm 60\%\ \,& 147\pm 39\% & 436\pm 75\%\ \,& \hfill - \hfill  \\
\hline
\end{array}}\] \caption{Experiment 1}\label{table1}
\end{table}

\begin{table}[htb]
\[{ \begin{array}{|c||r|r|r|r|}
\hline
n & \hfill \dim G \hfill & \hfill d \hfill & \hfill t_1 \hfill & \hfill t_2 \hfill \\
\hline
2^{1} & 1\qquad\quad & 1\qquad\quad &\ 0.108\pm 600\% & 0.28\pm 375\% \\
2^{2} & 1.81\pm 21\% & 2\qquad\quad & 0.144\pm 517\% & 0.308\pm 351\% \\
2^{3} & 2.96\pm 20\% &\ 3.68\pm 13\% & 0.212\pm 431\% & 0.444\pm 294\% \\
2^{4} & 4.44\pm 19\% & 6.39\pm 22\% & 0.344\pm 342\% & 0.756\pm 211\% \\
2^{5} & 6.12\pm 19\% & 10.4\pm 30\% & 0.716\pm 224\% & 1.88\pm 139\% \\
2^{6} & 8.05\pm 19\% & 16.2\pm 35\% & 1.4\pm 141\% & 7.05\pm 145\% \\
2^{7} & 10\pm 19\% & 24.4\pm 36\% & 3.08\pm 92\%\ \, & 35.8\pm 188\% \\
2^{8} & 12\pm 19\% & 34.1\pm 40\% & 6.96\pm 79\%\ \, & 184\pm 277\% \\
2^{9} & 14.5\pm 18\% & 51.1\pm 33\% & 17.5\pm 74\%\ \, & 1\,740\pm 322\% \\
2^{10} & 16.6\pm 17\% & 67.6\pm 34\% & 34.9\pm 81\%\ \, & 15\,700\pm 426\% \\
2^{11} & 18.5\pm 17\% & 85.2\pm 37\% & 64.4\pm 69\%\ \, & 37\,600\pm 166\% \\
2^{12} & 20.3\pm 15\% & 113\pm 37\% & 144\pm 88\%\ \, &\ 187\,000\pm 315\% \\
2^{13} &\ 23.3\pm 12\% & 131\pm 32\% & 209\pm 66\%\ \, & \hfill - \hfill \\
2^{14} & 25.8\pm 9\%\,\ & 189\pm 26\% & 593\pm 80\%\ \, & \hfill -\hfill \\
2^{15} & 28.8\pm 8\%\,\ & 268\pm 23\% & 1\,520\pm 60\%\ \, & \hfill -\hfill \\
\ 2^{16}\ & 30.9\pm 7\%\,\ & 446\pm 17\% & 6\,910\pm 55\%\ \, & \hfill -\hfill \\
\hline
\end{array}}\]\caption{Experiment 2}\label{table2}
\end{table}

The performance of \texttt{ElementProperty} depends essentially on the
size of $G$, while \texttt{Solvect} depends mostly on $n = \card{A}$
and to a lesser extent on $d = \dim \supspc$. We thus perform two sets
of experiments: the first in Table \ref{table1} is parametrized by
the size of $G$ (which is $2^{\dim G}$) and the second in Table
\ref{table2} is parametrized by $n$. In each case we measure the mean
values of $n$, $\dim G$ and $d$ as well as the times in milliseconds
taken by the two solvers ($t_1$ for \texttt{Solvect} and $t_2$ for
\texttt{ElementProperty}).

The samples are generated by choosing randomly the number and
dimensions of transitive constituents, i.e., a sequence $d_1,\ldots,d_q$
of strictly positive integers, then computing generators for the
transitive constituents and composing them randomly to produce generators 
for $G$. In the first experiment we guarantee that $G$ has the correct
dimension, bounded by $\max_{i=1}^q d_i \leq \dim G \leq
\sum_{i=1}^qd_i = d$. In the second experiment we guarantee that
$\sum_{i=1}^q 2^{d_i} = n$. 2-constraints are also generated randomly,
with the following bias: half of them are guaranteed to be satisfiable
(a solution is chosen randomly in $G$). Another bias has been
introduced: we choose the $d_i$ between 1 and 13, because computing
generators for an orbit of a size greater than $2^{13}$ takes too much
time.

Since our random samples are by no means supposed to be
representative, we also measure the standard deviation expressed as a
percentage of the mean value\footnote{When a process takes less than 4
  ms, GAP measures its duration as either 0 or 4 ms, hopefully with a
  probability depending on this duration. In that case the mean value
  should be accurate, but standard deviation is obviously
  exaggerated.}. We test 1000 instances on the low values and 100 on
the higher ones. Values are rounded to 3 digits. We see that
\texttt{ElementProperty} can hardly be used on groups of size much
bigger than $2^{20}$, while \texttt{Solvect} works well up to the
limits of the memory used by GAP (the limit is reached with $n=2^{17}$).

\section{Conclusion and perspectives}\label{sec-concl}

We can therefore solve $k$-constraints in the class of $\eF{p}$-vector
spaces in guaranteed polynomial time only when $k=p=2$, and we have
provided an algorithm to do so. For greater values of $k$ and $p$ the
problem is NP-complete, which is quite surprising considering the rich
structure of vector spaces and the relative simplicity of the
constraints that were considered. These results confirm how difficult
it can be to develop efficient algorithms for finding useful
symmetries.

However, our algorithm may be used on linear constraints regardless of
$k$ and $p$, and other experiments with \texttt{Solvect} suggests that
many constraints are linear. Furthermore, checking the linearity of the
constraint is fast.

But this still requires the group to be an elementary Abelian
$p$-group, which seems unlikely unless the problem under consideration
is of a geometric nature, for instance if a hypercube is involved (its
group of symmetries is an elementary Abelian 2-group). In general it
would be necessary to enforce this property by approximating the group
of symmetries by one or several elementary Abelian $p$-subgroups. This is
reasonable since symmetry pruning is meant to be fast, not complete
with respect to just any group of symmetries. It seems possible to
generalize these results to the class of Abelian groups in the line of
\cite{LuksR04}, at the expense of our elegant geometric interpretation of
linear constraints. We are currently investigating this generalization.

Identifying tractable restrictions of the generally intractable
problem of finding selected symmetries is therefore a natural approach
to efficient symmetry pruning. Methods using only special symmetries
have already been tried with some success, as in \cite{Peltier98}
where only transpositions are considered. We therefore believe the
present results open interesting perspectives.

\end{document}